\documentclass{article}

\usepackage[accepted]{icml2022}

\usepackage[utf8]{inputenc} 
\usepackage[T1]{fontenc}    
\usepackage{url}            
\usepackage{booktabs}       
\usepackage{amsfonts,amsmath,bm}       
\usepackage{nicefrac}       
\usepackage{microtype}      

\usepackage[utf8]{inputenc} 
\usepackage[T1]{fontenc}    
\usepackage{microtype}
\usepackage{graphicx}
\usepackage{longtable}
\usepackage{subcaption}
\usepackage{booktabs} 
\usepackage{amsfonts}
\usepackage{amsthm}
\usepackage{nicefrac}
\usepackage{enumitem}
\usepackage{balance}
\usepackage{xfrac}
\usepackage{setspace}
\usepackage{lscape}

\usepackage{hyperref}
\usepackage{color}

\usepackage{bm,amsmath,amssymb,amsfonts}
\usepackage{longtable,lscape}
\usepackage{booktabs}
\newtheorem{theorem}{Theorem}
\newtheorem{proposition}{Proposition}
\newtheorem{definition}{Definition}

\definecolor{mydarkblue}{rgb}{0,0.08,0.45}
\hypersetup{ 
    colorlinks,
    citecolor=mydarkblue,
    filecolor=mydarkblue,
    linkcolor=mydarkblue,
    urlcolor=mydarkblue 
}

\usepackage{url}
\usepackage{multirow}
\DeclareCaptionType{Algorithm}

\DeclareMathOperator*{\argmin}{argmin}

\newcommand{\indep}{\perp \!\!\! \perp}

\usepackage[utf8]{inputenc} 
\usepackage[T1]{fontenc}    
\usepackage{booktabs}       
\usepackage{amsfonts}       
\usepackage{nicefrac}       

\usepackage{algorithm}
\usepackage{algorithmic}
\usepackage{bm}
\usepackage{enumitem}

\newcommand{\bfx}{\mathbf{x}}
\newcommand{\bmx}{\bm{x}}

\newcommand{\mry}{\mathrm{y}}

\newcommand{\mra}{\mathrm{a}}

\newcommand{\bmg}{\bm{g}}
\newcommand{\bxg}{\mathbf{g}}

\newcommand{\ex}{\mathbb{E}}
\newcommand{\setx}{\mathbb{X}}
\newcommand{\seta}{\mathbb{A}}
\newcommand{\mrw}{\mathrm{w}}

\icmltitlerunning{End-to-End Balancing for Causal Inference}

\begin{document}

\twocolumn[
\icmltitle{End-to-End Balancing for Causal Continuous Treatment-Effect Estimation}



\icmlsetsymbol{equal}{*}

\begin{icmlauthorlist}
\icmlauthor{Mohammad Taha Bahadori}{yyy}
\icmlauthor{Eric J. Tchetgen Tchetgen}{yyy,comp}
\icmlauthor{David E. Heckerman}{yyy}
\end{icmlauthorlist}

\icmlaffiliation{yyy}{Amazon.com, Inc.}
\icmlaffiliation{comp}{Wharton School of the University of Pennsylvania}

\icmlcorrespondingauthor{Mohammad Taha Bahadori}{bahadorm@amazon.com}

\icmlkeywords{Machine Learning, ICML}

\vskip 0.3in
]

\printAffiliationsAndNotice{} 

\begin{abstract}
We study the problem of observational causal inference with continuous treatments in the framework of inverse propensity-score weighting. To obtain stable weights, we design a new algorithm based on entropy balancing that learns weights to directly maximize causal inference accuracy using end-to-end optimization. In the process of optimization, these weights are automatically tuned to the specific dataset and causal inference algorithm being used. We provide a theoretical analysis demonstrating consistency of our approach. Using synthetic and real-world data, we show that our algorithm estimates causal effect more accurately than baseline entropy balancing. 
\end{abstract}

\section{Introduction}
In many applications, including business, social sciences, and health sciences, we wish to infer the causal effect of a continuous treatment, such as the effect of drug dose or administration duration on a health outcome variable. Often, confounding factors exist that influence both treatment and response variable, making it necessary to account for their impact to yield accurate causal estimation. While methods for doing so have been well studied when treatments are binary, 
causal inference with continuous treatments is far more challenging and largely understudied \citep{galagate2016causal,ai2021estimation}. This is so primarily because continuous treatments induce  uncountably many potential outcomes per unit (e.g., a subject), only one of which is observed for each unit and across units.

Propensity score weighting \citep{robins2000marginal,imai2004causal}, stand-alone or combined with regression-based models to achieve double robustness \citep{diaz2013targeted,kennedy2017nonparametric,nie2020vcnet}, has become the state of the art for causal inference.
If the weights, inversely proportional to the conditional distribution of the treatment given the confounders, are correctly modeled, the weighted population will appear to come from a randomized study.
However, this approach faces several challenges. One, the weights only balance the confounders in expectation, not necessarily in the given data \citep{zubizarreta2011matching}. Two, the weights can be very large for some units, leading to unstable estimation and  uncertain inference. 
As a possible remedy, entropy balancing \citep{hainmueller2012entropy} proposes the use of an alternate set of weights such that they are less extreme but still have the property that the weight data appears to come from a randomized study. 

In this work, we note that balanced weights do not directly optimize the accuracy of causal inference, and we introduce an alternative approach that does. We propose \textit{End-to-End Balancing} (E2B), which uses synthetic data resembling the given data and end-to-end optimization to identify these weights. The E2B weights are thus customized for the given dataset and causal inference algorithm. Because we do not know the true treatment response function in the real data, we propose a new approach to generate synthetic training datasets and show that this approach is robust to model misspecification.

To theoretically analyze end-to-end balancing, we define \textit{Generalized Stable Weights} (GSW) for causal inference as a generalization of the stable weights proposed by \citet{robins2000marginal}. We prove that weights learned by entropy balancing for continuous treatments, including E2B weights, are unbiased estimators of GSWs. We also show that E2B weights are asymptotically consistent and efficient estimators of the population weights.

We perform three sets of experiments to demonstrate accuracy improvements by E2B. Two experiments with synthetic data, one with linear and another with non-linear response functions, and show that E2B is more accurate than baseline entropy balancing and inverse propensity score techniques. We also study the impact of mis-specification in the synthetic data generation process. In the experiments on real-world data, we qualitatively evaluate the average treatment effect function learned by E2B. We also qualitatively analyze the patterns in E2B weights.

\section{Problem Definition and Related Work}
\paragraph{Problem Statement.} Suppose we have the triplet of $(\bfx, \mra, \mry)$, where $\bfx \in \setx \subset \mathbb{R}^{r}$, $\mra \in \seta \subset \mathbb{R}$ and $\mry \in \mathbb{R}$ denote the confounders, treatments, and response variables, respectively, from an observational causal study. In our continuous treatment setting \citep[Ch. 1.2.6]{galagate2016causal}, we denote the value of $\mathrm{y}$ after intervening by setting the treatment to $a$ as $\mathrm{y}^{(a)}$.  This quantity is commonly known as a {\em potential outcome}..   Given an i.i.d. sample of size $n$, $\{(\bmx_i, a_i, y_i)\}_{i=1}^{n}$, our objective is to eliminate the impact of the confounders and identify the average treatment effect function $\mu(a) = \ex[\mry^{(a)}]$, which is also called the response function. We make the two classic assumptions: (1) Strong ignorability: $\mry^{(\mra)} \indep \mra~|~\bfx$. (i.e., no hidden confounders) and (2) Positivity: $0 < P(\mra | \bfx) < 1$. 

\paragraph{General Causal Inference Literature.} The literature on causal inference is vast and we refer the reader to the books for the general inquiry \citep{pearl2009causality,imbens2015causal,spirtes2000causation,peters2017elements}. Instead, we focus on reviewing the inference techniques for \textit{continuous} treatments. In particular, we narrow down our focus on propensity score weighting approaches \citep{robins2000marginal,imai2004causal}, because they can either be used alone or combined with the regression algorithms to create doubly robust algorithms.

\paragraph{Causal Inference via Weighting.} In inverse propensity score weighting, a popular approach for causal inference, we create a pseudo-population by weighting data points such that, in this population, the confounders and treatments are independent. Thus, regular regression algorithms can estimate the causal response curve using the pseudo-population, which resembles data from randomized trials. Throughout this paper, we will denote the parameters of the pseudo-population with a tilde mark---for example, $\widetilde{f}(a)$ denotes the marginal density of treatments in the weighted population. Multiple forms of propensity scores have been proposed for continuous treatments \citep{hirano2004propensity,imai2004causal}. 
The commonly-used \textit{stablized weights} \citep{robins2000marginal,zhu2015boosting} are defined as the ratio of marginal density over the conditional density of the treatments: $sw = \sfrac{f(a)}{f(a|\bmx)}$. 

\paragraph{Problems with Propensity Scores.} \citet{zubizarreta2011matching} list two challenges with the propensity scores: One, the weights only balance the confounders in  expectation, not necessarily in the given data. (Two, the weights can be very large for some of the data points, leading to unstable estimations.  The challenges are amplified in the continuous setting because computing the stabilized weights requires correctly choosing two models, one for the marginal and one for the conditional distributions of the treatments. \citet{kang2007demystifying} and \citet{smith2005does} provide multiple pieces of evidence that the propensity score methods can lead to large biases in the estimations.  \citet{robins2007comment} propose techniques to fix the large weights problems in the binary treatment examples discussed by \citet{kang2007demystifying}\citep{li2018balancing,zhao2019covariate} describe techniques to learn more robust propensity scores for binary treatments. The case of continuous treatments, however, has received considerably less attention.

\paragraph{Entropy Balancing.} To address the problem of extreme weights, \textit{Entropy Balancing (EB)} \citep{hainmueller2012entropy} estimates weights such that they balance the confounders subject to a measure of dispersion on the weights to prevent extremely large weights. Other loss functions using different dispersion metrics have been proposed for balancing \citep{zubizarreta2015stable,chan2016globally}. \citet{zhao2016entropy} show that the entropy balancing is doubly robust. Entropy balancing has been extended to the continuous treatment setting  \citep{fong2018covariate,vegetabile2021nonparametric}, where the balancing condition ensures that weighted correlations between the confounders and the treatment are zero.  \citet{ai2021estimation} propose a method for estimating the counterfactual distribution in the continuous treatment setting.

\section{Methodology}
To describe our end-to-end balancing algorithm, we first need to describe entropy balancing for continuous treatments with base weights.

\subsection{Entropy Balancing for Continuous Treatments}
\paragraph{Causal Inference via Entropy Balancing.} Entropy balancing creates a pseudo-population using instance weights $w_i\in \mathbb{R}_{>0}, i=1, \ldots, n$, in which the treatment $\mra$ and the confounders $\bfx$ are approximately independent from each other. The independence is enforced by first selecting a set of functions on the confounders $\phi_k(\cdot): \setx\mapsto \mathbb{R}$, for $k=1, \ldots, K$, that are dense and complete in $L^2$ space. Given the $\phi$ functions, we approximate the independence relationship by $\widehat{\mathbb{E}}_n[\mra \phi_k(\bfx)] = 0$, for $k=1, \ldots, K$, where the empirical expectation $\widehat{\mathbb{E}}_n$ is performed on the pseudo-population. Hereafter, we will denote the mapped data points as $\bm{\phi}(\bmx_i) = [\phi_1(\bmx_i), \ldots, \phi_K(\bmx_i)]$. The $\phi_k(\cdot)$ functions can be chosen based on prior knowledge or defined by the penultimate layer of a neural network that predicts $(\mra, \mry)$ from $\bfx$. Our contributions in this paper are orthogonal to the choice of the $\phi_k(\cdot)$ functions and can benefit from ideas on learning these functions \citep{zeng2020double}. The data-driven choice of the number of bases $K$ is beyond the scope of current paper and left to future work.

\paragraph{Balancing Constraint for the Continuous Treatments.} Following \citep{fong2018covariate,vegetabile2021nonparametric}, in the case of continuous treatments, we first de-mean the confounders $\bm{\phi}(\bmx_i)$ and treatments $\mra$ such that without loss of generality they are taken to have mean zero. The balancing objective is to learn a set of weights $w_i, i=1, \ldots, n$ that satisfy $\sum_{i=1}^{n}w_i\bm{\phi}(\bmx_i)=\bm{0}$, $\sum_{i=1}^{n}w_ia_i=0$, and $\sum_{i=1}^{n}w_ia_i\bm{\phi}(\bmx_i)=\bm{0}$. We can write these three constraints in a compact form by defining a $(2K+1)$--dimensional vector $\bmg_i = [\bm{\phi}(\bmx_i), a_i, a_i\bm{\phi}(\bmx_i)]$. The constraints become 
$\sum_{i=1}^{n}w_i\bmg_i=\bm{0}$. We stack the $\bmg$ vectors in a $(2K+1)\times n$ dimensional matrix $\bm{G}$ for compact notation. 
In this work, without loss of generality, we will present our idea with first order balancing, with neither higher order moments  \citep{galagate2016causal} nor balancing in the kernel space \citep{wong2018kernel,kallus2019kernel,hazlett2018kernel}.

\paragraph{Primal and Dual EB.} A variety of dispersion metrics have been proposed as objective function for minimization, including entropy and variance of the weights  \citep{wang2020minimal}. \citet{hainmueller2012entropy} originally proposed minimizing the KL-divergence between the weights and a set of {\em base weights} $q_i, i=1, \ldots, n$. We discuss the details on choice of base weights below, however, note that $q_i=\mathrm{constant}$ leads to maximization of the entropy of weights. Using this dispersion function and the balancing constraints, entropy balancing optimization is as follows:
\begin{align}
    \widehat{\bm{w}}=& \argmin_{\bm{w}}~\sum_{i=1}^nw_i\log\left(\frac{w_i}{q_i}\right),\label{eq:primal}\\
    \text{s.t.   } \qquad
    &\text{(i)}~  \bm{G}\bm{w} = \bm{0}, \nonumber\\
    &\text{(ii)}~ \bm{1}^{\top}\bm{w} = 1,\nonumber\\
    &\text{(iii)}~ w_i \geq 0 \text{ for } i=1, \ldots, n.\nonumber
\end{align}

The above optimization problem can be solved efficiently using its Lagrangian dual:
\begin{align}
    \widehat{\bm{\lambda}} &= \argmin_{\bm{\lambda}}\; \log\left(\bm{1}^{\top}\exp\left(- \bm{\lambda}^{\top}\bm{G} +\bm{\ell}\right) \right),
    \label{eq:dual}
\end{align}
where $\ell_i = \log q_i$ are the \textit{log-base-weights}. Given the solution $\widehat{\bm{\lambda}}$, the balancing weights can be computed as $\bm{w} = \mathrm{softmax}\left(- \widehat{\bm{\lambda}}^{\top}\bm{G} +\bm{\ell}\right)$, where $\mathrm{softmax}(\bm{v}) = \sfrac{\exp \bm{v}}{\left(\bm{1}^{\top}\exp \bm{v}\right)}$ for any vector $\bm{v}$. We select the mapping dimension $K$ such that problem (\ref{eq:dual}) is well-conditioned and leave the analysis of the high dimensional setting $K \approx n$ to future work. We can also add an $L_1$ penalty term to the dual objective in Eq. (\ref{eq:dual}), which corresponds to approximate balancing \citep{wang2020minimal}. 

An observation that is key to our approach is that each log-base-weight $\ell_i$ is a degree of freedom that we have in the Eq. (\ref{eq:dual}) to improve the quality of causal estimation. In the next section we propose to parameterize the log-base-weights and optimize them for causal inference using synthetic data. Our analysis in Section \ref{sec:analysis} shows that with arbitrary base weights, causal estimation using the weights learned in Eq. (\ref{eq:dual}) will be consistent.

\subsection{Learning the Base Weights}
\citet{hainmueller2012entropy} suggests two approaches for choosing base weights: (1) weights obtained from a conventional propensity score model and (2), in the context of survey design, using knowledge about the sampling design. We argue that a data-driven approach that learns customized base weights for a given dataset and weighted causal regression algorithm can further improve performance. For example, when we use weighted linear regression for causal inference, appropriate base weights can decrease the condition number of the design matrix and thus improve the quality of the regression. From another point of view, minimizing the KL-divergence between the weights $w$ and the base weights $q$ can act as a regularizer and improve the accuracy of weight estimation. Finally, as we will show in our algorithm, we can embed our prior distribution over possible response functions in the base weights and  improve the causal inference accuracy.


To address this problem, we define the log-base-weights $\bm{\ell}$ as a parametric function (e.g., a neural network) of the treatments and confounders---that is, $\ell_{\bm{\theta}}(\bmx, a)$. We learn the base weights with the goal of improving the accuracy of the subsequent weighted regression. This is challenging because simply optimizing the weighted regression loss (e.g., weighted MSE) leads to degenerate results. That is, learning $\bm{\ell}$ to minimize the loss in prediction of the outcome $\mry$ will lead to exclusion of the difficult-to-predict data points from the regression, which is undesirable. Thus, we need to find another loss function to optimize, ideally a loss function that directly minimizes the error in estimation of the response function $\mu(a)$.


\begin{algorithm}[tb]
  \caption{Stochastic Training of $\ell_{\bm{\theta}}$ for End-to-End Balancing}
  \label{alg:nn}
\begin{algorithmic}[1]
\REQUIRE Data tuples $(\bmx_i, a_i, y_i)$ for $i=1, \ldots, n$ with an unknown response function $\mu(a)$.
\REQUIRE Representation functions $\bm{\phi}(\cdot)$ and $\bm{\psi}(\cdot)$, split size $n_1 < n$ and batch size $B$.
\STATE Generate a random set of indexes $I, |I| = n_1$ and its complement $I^{c}$ and split the data to $S$ and $S^c$ using them.
\STATE Estimate the distribution of noise in $\mry$ given $(\mra, \bfx)$ as $\widehat{F}_{\varepsilon}$. \label{line:noise}
\STATE Compute $\bm{G}$ by stacking $\bmg_i = [\bm{\phi}(\bmx_i), a_i, a_i\bm{\phi}(\bmx_i)]$, for $i=1, \ldots, n$.
\FOR{Number of Iterations}
\STATE Generate $B$ datasets  $\{(\bmx_i, a_i, \overline{y}_{i, b})\}_{i=1}^{n}$ for $b=1, \ldots, B$ using $\varepsilon\sim \widehat{F}_{\varepsilon}$, and randomly selected $\overline{\mu}(a)_b$ response functions. \label{line:rand}
\STATE $\ell_{i} \gets \ell_{\bm{\theta}}(\bm{\psi}(a_i, \bm{x}_i))$.
\STATE \small$\widehat{\bm{\lambda}} \gets \argmin_{\bm{\lambda}}\;\left\{ \log\left(\bm{1}^{\top}\exp\left(- \bm{\lambda}^{\top}\bm{G} +\bm{\ell}\right)\right)+\gamma\|\bm{\lambda}\|_1\right\}$\normalsize\quad using only $S$ data.\label{line:dual}
\STATE $\bm{w} \gets \mathrm{softmax}\left(- \widehat{\bm{\lambda}}^{\top}\bm{G} +\bm{\ell}\right)$ using only $S^c$ data. \label{line:softmax}
\STATE $\widehat{\mu}(a)_b\gets$ weighting-based causal estimates using $(a_i, \overline{y}_{i,b}, w_{i})$ in $S^c$ for $b=1, \ldots, B$. \label{line:regress}
\STATE Take a step in $\bm{\theta}$ to minimize $\frac{1}{B}\sum_{b=1}^{B}\left(\widehat{\mu}(a)_b - \overline{\mu}(a)_b\right)^2$. \label{line:loss}
\ENDFOR
\RETURN The $\ell_{\widehat{\bm{\theta}}}$ function.
\end{algorithmic}
\end{algorithm}

Our idea for learning the parameters of the base weights is to generate multiple pseudo-responses $\overline{\mry}$ with randomly generated  response functions $\overline{\mu}(a)\sim P_{\overline{\mu}}$. We will elaborate on the possible choices for the distribution $P_{\overline{\mu}}$ later in this section. Now that we know the true response function $\overline{\mu}(a)$ in the randomly generated data, we can perform causal inference and obtain the estimation of the known response curve $\widehat{\mu}(a)$ using our weights. Algorithm \ref{alg:nn} outlines our stochastic training of the log-base-weight function. We also schematically visualize the algorithm in Figure \ref{fig:diagram}. In Step \ref{line:noise}, we estimate the distribution of noise using the residuals of regressing $\mry$ over $(\mra, \bfx)$, capturing the possible heteroskedasticity in the noise.  Then, in each iteration, we draw a batch of possible datasets. To generate each dataset, we randomly choose a response function $\overline{\mu}(a)$ and use it to generate the entire dataset (see Section \ref{sec:synth} for examples of random functions). For the entire batch, we use $\ell_{\bm{\theta}}$ to learn the log-base-weights, and subsequently learn the weights in lines \ref{line:dual}--\ref{line:softmax}. In line \ref{line:regress} we use a weighted regression algorithm to find our estimation $\widehat{\mu}(a)$ of the randomly-generated $\overline{\mu}(a)$. Our loss function is the mean squared error between the latter quantities. Our algorithm, which we call \textit{End-to-End Balancing} (E2B),  gets its name from this end-to-end optimization.

\begin{figure}
    \centering
    \includegraphics[width=0.45\textwidth]{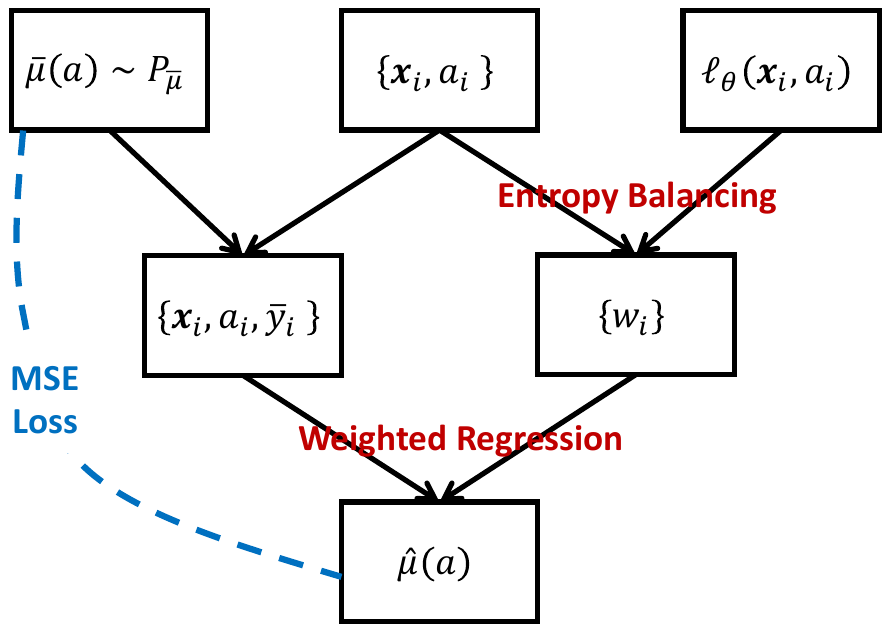}
    \caption{Overview of the steps in End-to-End Balancing.}
    \label{fig:diagram}
\end{figure}

\paragraph{Sample Splitting.} The E2B procedure involves estimation of two sets of parameters $\bm{\theta}$ in the $\ell_{\bm{\theta}}$ and $\bm{\lambda}$ for entropy balancing. The joint estimation of $\bm{\theta}, \bm{\lambda}$ on a single sample will result in bias \citep{chernozhukov2018double}. Thus, we split the sample to two mutually exclusive parts and perform the optimizations on separate partitions of data.

\paragraph{Distribution of Random Response Functions $\bm{P_{\overline{\mu}}}$.} Ideally, we should rely on domain experts or rough ideas from past studies for choosing the random set of response functions $\overline{\mu}(a)$ that includes the true response function. This will help us embed our prior knowledge in the base weights and improve the accuracy of causal inference.
Alternatively, we can choose broad function classes such as random piecewise smooth functions or polynomial functions with random coefficients. See Section \ref{sec:exp} for examples of generating random response functions. We can also use generative models to generate data that is more similar to our sample \citep{athey2021using}. 

\paragraph{Features Fed to $\ell_{\bm{\theta}}$.} We can feed the raw values of the treatments and any handcrafted features, denoted by $\bm{\psi}(a_i, \bm{x}_i)$. We empirically find that $\bm{\psi}(a_i, \bm{x}_i) = (\log p(a_i), \log p(a_i|\bmx_i))$ makes training the $\ell_{\bm{\theta}}$ easier. These features are the logarithms of the numerator and denominator of the stable weights. Given this choice, we can visualize the $\ell_{\bm{\theta}}$ function and find its relationship to the marginal and conditional distributions of the treatment. We describe the details of our neural network model for $\ell_{\bm{\theta}}$ and our techniques for training in Appendix \ref{sec:details}. 

\paragraph{Weighted Regression Algorithms.} To be able to differentiate the loss function with respect to $\bm{\theta}$, we need weighted regression algorithms whose estimates are differentiable with respect to the weights. In the linear average treatment effect function we choose weighted linear regression and in the non-linear setting we use the weighted polynomial regression and the local kernel regression, as used by \citet{flores2012estimating}.

\paragraph{Double Robustness.} \citet{zhao2016entropy} show that, in the binary case, entropy balancing is doubly robust. We do not attempt to show double robustness for E2B because we see E2B as a meta algorithm that learns customized weights for each dataset and algorithm. We can either (1) plug-in the E2B weights in the doubly robust algorithm and expect improved accuracy, or (2) learn weights that directly minimize the error of doubly robust algorithms such as \citep{diaz2013targeted,kennedy2017nonparametric}. However, if we use an outcome regression model for generating the random response functions $\overline{\mu}(a)$, the E2B weights may no longer provide significant improvements to the doubly robust techniques.

\section{Analysis}
\label{sec:analysis}
We prove that for any arbitrary choice of the log-base-weight function $\ell_{\bm{\theta}}$, our approach consistently estimates causal effects. Before proving the consistency results, we characterize the quantity that our solution converges to. 

\begin{definition} \textbf{Generalized Stable Weights.} Suppose $f(a,\bmx)$ denote the joint probability density function of treatments and confounders in a population. Suppose $\widetilde{f}(a)$ and $\widetilde{f}(\bmx)$ denote two arbitrary density functions, possibly different from the marginal density functions in our population, that satisfy $\ex_{\bfx \sim \widetilde{f}(\bmx)}[\bfx]=\bm{0}$ and $\ex_{\mra \sim \widetilde{f}(a)}[\mra]=0$.  We define the Generalized Stable Weights as follows
\begin{equation}
    w_{GSW}(a, \bmx) = \frac{\widetilde{f}(a)\widetilde{f}(\bmx)}{f(a, \bmx)}.
    \label{eq:def}
\end{equation}
\end{definition}
\paragraph{Remark.} Our definition generalizes the stabilized weights defined by \citet{robins2000marginal} in that their weights require $\widetilde{f}(a)$ and $\widetilde{f}(\bmx)$ match the marginal probability density functions in the original population, whereas GSW allows these densities to be less constrained. The next proposition motivates the generalization.

\begin{proposition}
\label{lem:gsw}
The generalized stable weights $w_{GSW}$ satisfy $\ex\left[w_{GSW}\mra\bfx\right]= \mathbf{0}$.
\label{prop:gsw}
\end{proposition}
\begin{proof}
\begin{align*}
    \ex\left[\mrw_{GSW}\mra\bfx\right] &= \ex\left[\frac{\widetilde{f}(\mra)\widetilde{f}(\bfx)}{f(\mra, \bfx)}\mra\bfx\right] \\
    &= \int\int\frac{\widetilde{f}(a)\widetilde{f}(\bmx)}{f(a, \bmx)}a\bmx \ \mathrm{d}F_{\mra,\bfx}(a, \bmx) \\
    & = \int{a\widetilde{f}(a)\mathrm{d}a} \int{\bmx\widetilde{f}(\bmx) \ \mathrm{d}\bmx} = \mathbf{0},
\end{align*}
where the last equation holds due to the zero mean assumption for the $\widetilde{f}(a)$ and $\widetilde{f}(\bmx)$ distributions.
\end{proof}

Now, we can show that with an appropriate choice of the $\bm{\phi}$ functions, the solution of Eq. (\ref{eq:dual}) approximates the generalized stable weights. Consider the population version of Eq. (\ref{eq:dual}):
\begin{equation}
    \bm{\lambda}^\star = \argmin_{\bm{\lambda}} \log\left( \ex\left[\exp(\bxg^\top\bm{\lambda} + \ell)\right]\right).
\label{eq:pop}
\end{equation}
The weights corresponding to $\bm{\lambda}^\star$ can be calculated as $w^\star = C\exp(\bxg^\top\bm{\lambda}^\star + \ell)$, where $C=\left(\int \exp(\bxg^\top\bm{\lambda}^\star + \ell) \ \mathrm{d}F(a, \bmx)\right)^{-1}$ is the normalization constant.
\paragraph{Assumptions.}
\begin{enumerate}
    \item $f(a, \bmx) \geq c >0$ for all $(a, \bmx) \in \seta \times \setx$ pairs, where $c$ is a constant.
    \item Suppose the basis functions are dense and rich enough such for some small values of $\delta_{\bm{\phi}_K}$ that they satisfy:
\begin{equation*}
    \mathbb{E}[\mathrm{a}\bm{\phi}(\mathbf{x})] = \bm{0}  \; \text{only if} \; \sup_{a, \bm{x}}\left|f(a, \bm{x})-f(a)f(\bm{x}) \right| = \delta_{\bm{\phi}_K}.
\end{equation*}
\item Suppose the population problem in Eq. (\ref{eq:pop}) has a unique solution $\bm{\lambda}^{\star}$ and the corresponding weights are denoted by $w^{\star}$.
\end{enumerate}
The following theorem shows that the solution to Eq. (\ref{eq:pop}) converges to $w_{GSW}$:
\begin{theorem} Given the assumptions, the solution to the population problem satisfies:
\begin{equation}
    \sup_{a, \bmx}\left| w^{\star}(a, \bmx) - w_{GSW}(a, \bmx) \right| \leq \delta_{\bm{\phi}_K}/c.
\end{equation}
\label{thm:bias}
\end{theorem}
If we select the function set $\bm{\phi}_K$ such that $\delta_{\bm{\phi}_K} = o(1)$, the theorem shows that $w^{\star}(a, \bmx)$ is an unbiased estimator of $w_{GSW}(a, \bmx)$.
Notice that Assumption 1 is only slightly stronger than the common positivity assumption. Assumption 2 requires us to select the mapping functions such that zero correlation between the mapped confounders and the treatment implies their independence, approximately. We provide the proof in Appendix \ref{sec:proof1}.

Note that the quality of the $\bm{\psi}$ features and neural network training does not affect the unbiasedness of the E2B because of the balancing constraint is still satisfied. 
The flexibility in choice of $\widetilde{f}$ distributions in the definition of $w_{GSW}$ is due to the fact that we require only first order balancing, enforcing zero correlation between the first order moments of the treatment and confounders. If we enforce higher order balancing constraints in the form of $\mathbb{E}[\mathrm{w}^{\star}\phi_1(\mathrm{a})\phi_2(\bfx)] = \mathbb{E}[\mathrm{w}^{\star}\phi_1(\mathrm{a})]\,\cdot\,\mathbb{E}[\mathrm{w}^{\star}\phi_2(\bfx)]$ for any suitable functions $\phi_1$ and $\phi_2$,  Theorem 1 in \citep{ai2021estimation} shows that $w^{\star} = \sfrac{f(a)}{f(a|\bmx)}$. The more flexible form of weights in Eq. (\ref{eq:def}) allows us to pick the marginals $\widetilde{f}(a)$ and $\widetilde{f}(\bmx)$ with more freedom. In this work, we have chosen a data-driven way to learn them.

Finally, the following theorem establishes the asymptotic consistency and normality result for each individual weight estimated by E2B, under the common regularity conditions for problem (\ref{eq:dual}).
\begin{theorem} Suppose $\Lambda \subset \mathbb{R}^{2K+1}$ is an open subset of Euclidean space and the solution $\widehat{\bm{\lambda}}_n \in \Lambda$ to Eq. \eqref{eq:dual} is within the subset. The weights estimated by Eq. \eqref{eq:dual} are asymptotically normal for $i=1, \ldots, n$:
\begin{align}
\sqrt{n}\left(\,\widehat{w}_n(a_i, \bmx_i) - w^{\star}(a_i, \bmx_i)\,\right) ~ &\overset{d}{\to} ~ \mathcal{N}(0, \sigma^2(a_i, \bmx_i)).\label{eq:cons_w}
\end{align}
We provide the population form of $\sigma^2(a_i, \bmx_i)$ and an unbiased sample estimate for  it in  Appendix \ref{sec:proof2}.
\label{thm:cons}
\end{theorem}

\section{Experiments}
\label{sec:exp}
We use two synthetic and one real-world datasets to evaluate the performance of E2B. In the synthetic datasets, we have access to the true treatment effects; thus we measure accuracy of the algorithms in recovering the treatment effects. In the real-world data, we qualitatively evaluate the estimated causal treatment effect curve and inspect the learned log-base-weight function.

\paragraph{Baselines.} One baseline in our study is the Inverse Propensity score Weighting (IPW) with Stable Weights \citep{robins2000marginal}, perhaps the most commonly used technique. To avoid extreme weights and prevent instability, we trim (Winsorize) the weights by $[5, 95]$ percentiles \citep{cole2008constructing,crump2009dealing,chernozhukov2018double}. Another baseline in our experiments is Entropy Balancing \citep{vegetabile2021nonparametric}, which is equal to E2B with $\ell_{\bm{\theta}}=\mathrm{constant}$. EB allows us to do an ablation study and see the exact amount of improvement when learning a customized $\ell_{\bm{\theta}}$ function. We also include EB with the stabilized weights (SW) as base weights ($\ell_{\bm{\theta}}=\log \widehat{p}(a) - \log  \widehat{p}(a|\bmx)$). For fair comparison to the entropy balancing methods, we only include first order balancing methods, as our idea of learning the base-weights can be combined with the higher order and kernel-based balancing methods.
 Finally, we also include the permutation weighting algorithm \citep{pmlr-v139-arbour21a} that proposes to compute the weights using permutations of the treatments and a classifier that predicts the probability of being permuted. We provide further details on this algorithm in Appendix \ref{sec:perm}. 

\begin{table*}[t]
    \centering
    \caption{Average RMSE for estimation of the response functions. The results are in the form of ``mean (std. err.)'' from 100 runs.}
    \label{tab:synthetic}
    \begin{tabular}{l|c|c}
    Algorithm & Linear & Non-linear \\
    \toprule
    Inverse Propensity Weighting (SW)& $2.057~(0.437)$ & $0.530 ~ (0.025)$\\
    Permutation Weighting&
    $1.1543~(6.580)$ & $0.525~(0.250)$\\
    Entropy Balancing  (Const.)   & $0.626 ~(0.041)$ & $0.417 ~ (0.024)$\\
    Entropy Balancing  (SW)   & $0.412 ~ (0.029)$ & $0.399 ~(0.021)$\\
    End-to-End Balancing     & $\mathbf{0.299 ~(0.029)}$ & $\mathbf{0.249 ~(0.011)}$\\
    \bottomrule
    \end{tabular}
\end{table*}

\begin{figure*}[t]
    \centering
    \begin{subfigure}[t]{0.48\textwidth}
    \centering
    \includegraphics[width=\textwidth]{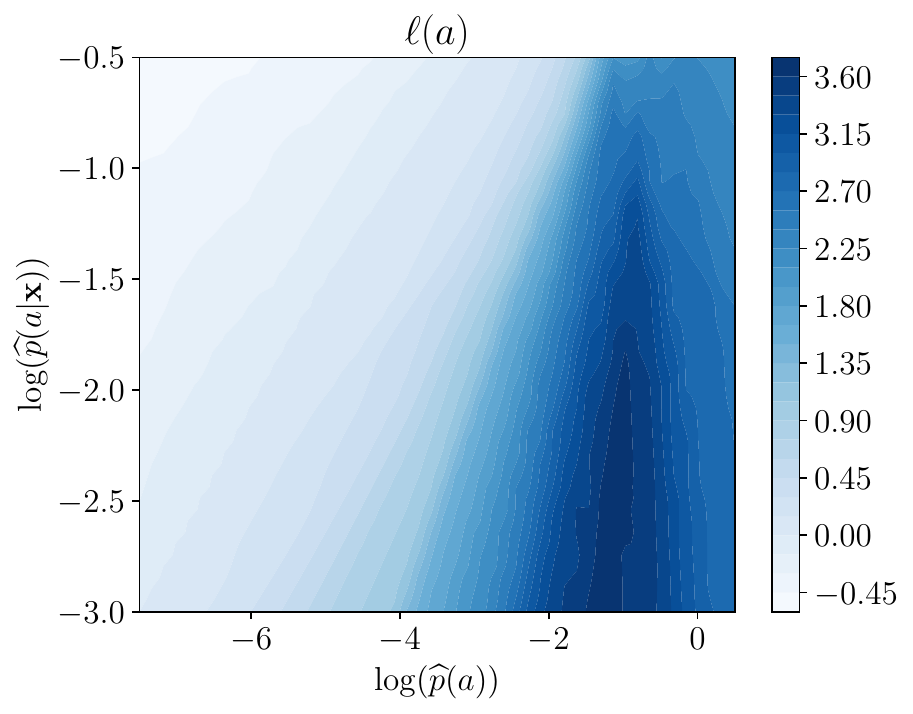}
    \caption{Linear Design}
    \label{fig:synthetic_linear}
    \end{subfigure}
    \quad
    \begin{subfigure}[t]{0.48\textwidth}
    \centering
    \includegraphics[width=\textwidth]{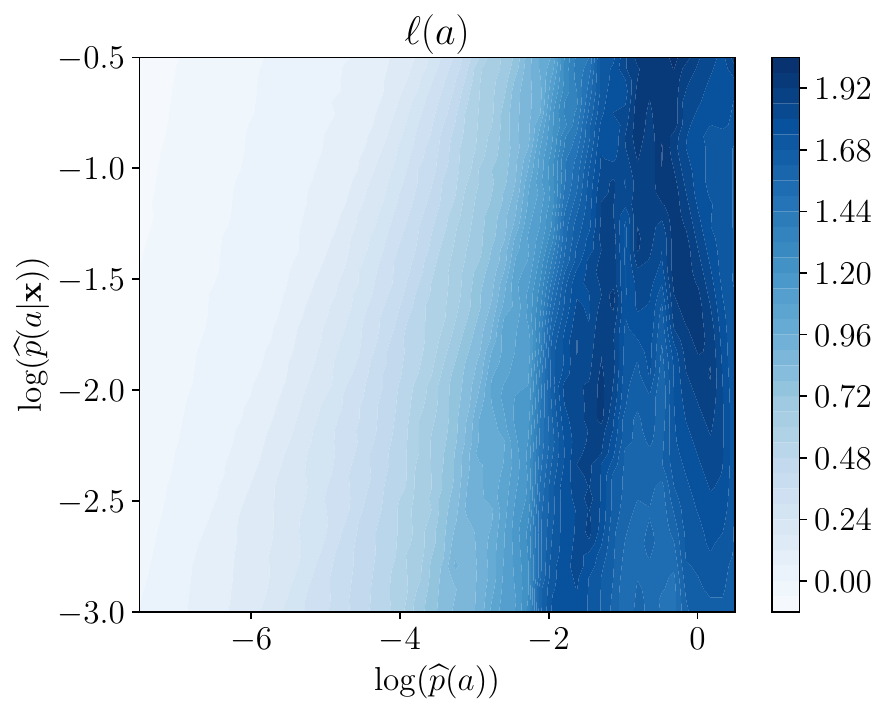}
    \caption{Non-linear Design}
    \label{fig:synthetic_nonlinear}
    \end{subfigure}
    \caption{The estimated log-base-weight function $\ell_{\bm{\theta}}$ as a function of logarithms of the empirical density of the treatment $\log(\widehat{p}(a))$ and conditional distribution $\log(\widehat{p}(a|\bfx))$. We perform the experiment 100 times and report the median and the inter-quantile range. We align all curves by normalizing their value at the beginning to zero.}
    \label{fig:log-base}
\end{figure*}

\paragraph{Training Details.} All neural networks are trained using Adam \citep{kingma2014adam} with early stopping based on validation error. The learning rate and architectural parameters of the neural networks are tuned via hyperparameter search on the validation data.  We provide the details of the neural networks used for the $\ell_{\bm{\theta}}$ and propensity score estimation for IPW in Appendix \ref{sec:details}. 

\paragraph{Regularization.} Proper tuning of the regularization term $\gamma$ in Line \ref{line:dual} of Algorithm \ref{alg:nn} is critical to the performance of EB and E2B. Our first observation is that using $L_2$ regularization allows more stable tuning of the regularization coefficient $\gamma$, though the final accuracies do not significantly change. We tune the regularization coefficient for EB by measuring the bootstrapped covariate balance, as described by \citet{wang2020minimal}. However, for E2B, we have access to the validation error on the pseudo-datasets and use the validation error to tune $\gamma$. In the experiments, we observe that tuning with validation error leads to smaller values for the penalization coefficient $\gamma$ and more accurate causal inference.

\subsection{Synthetic Data Experiments}
\label{sec:synth}
\paragraph{Linear.}
We use the following steps to generate 100 datasets, each with 1000 data points.
\begin{enumerate}
    \item Generate confounders $\bfx \in \mathbb{R}^{5}$, $\bfx \sim \mathcal{N}(\bm{0}, \bm{\Sigma})$, where $\bm{\Sigma}$ is a tridiagonal covariance matrix with diagonal and off-diagonal elements equal to $1.0$ and $0.2$, respectively.
    \item $\mra \sim \mathcal{N}(\mu_a, 0.3^2)$, where $\mu_a = \sin(\bm{\beta}_{xa}^{\top}\bfx)$ and $\beta_{xa, k} \sim \mathrm{Unif}(-1, 1)$ for $k=1, \ldots, 5$.
    \item $\mry \sim \mathcal{N}(\mu_y, 0.5^2)$, where $\mu_y = \bm{\beta}_{xy}^{\top}\bfx + \beta_{ay}\mra$, where $\beta_{ax}, \beta_{xy, k} \sim \mathcal{N}(0, 1)$ for $k=1, \ldots, 5$.
\end{enumerate}
We use weighted least squares as the regression algorithm and report the average $|\widehat{\beta_{ay}} - \beta_{ay}|$ over all 100 datasets.

\paragraph{Nonlinear}
We first generate confounders $\bfx$ and treatments $\mra$ similar to steps 1 and 2 of the linear case. Then, we generate the response variable according to $\mry \sim \mathcal{N}(\mu_y, 0.5^2)$, where $\mu_y = \bm{\beta}_{xy}^{\top}\bfx + h_{\bm{\gamma}_{ay}}(\mra)$, where $\beta_{xy, k} \sim \mathcal{N}(0, 1)$ for $k=1, \ldots, 5$, and $h_{\bm{\gamma}}(z) = \gamma_0 + \gamma_1 z + \gamma_2 (z^2-1) + \gamma_3 (x^3-3x)$ are Hermit polynomials.  Similar to the linear case, we generate 100 samples of size 1000. We use the weighted polynomial regression as the regression algorithm to estimate $\widehat{\bm{\gamma}}$ and report the average RMSE between true $\bm{\gamma}$ and $\widehat{\bm{\gamma}}$. We report the mean and standard error of errors on 100 datasets in Table \ref{tab:synthetic}.

\begin{figure*}[t]
    \centering
    \begin{subfigure}[t]{0.48\textwidth}
    \centering
    \includegraphics[width=\textwidth]{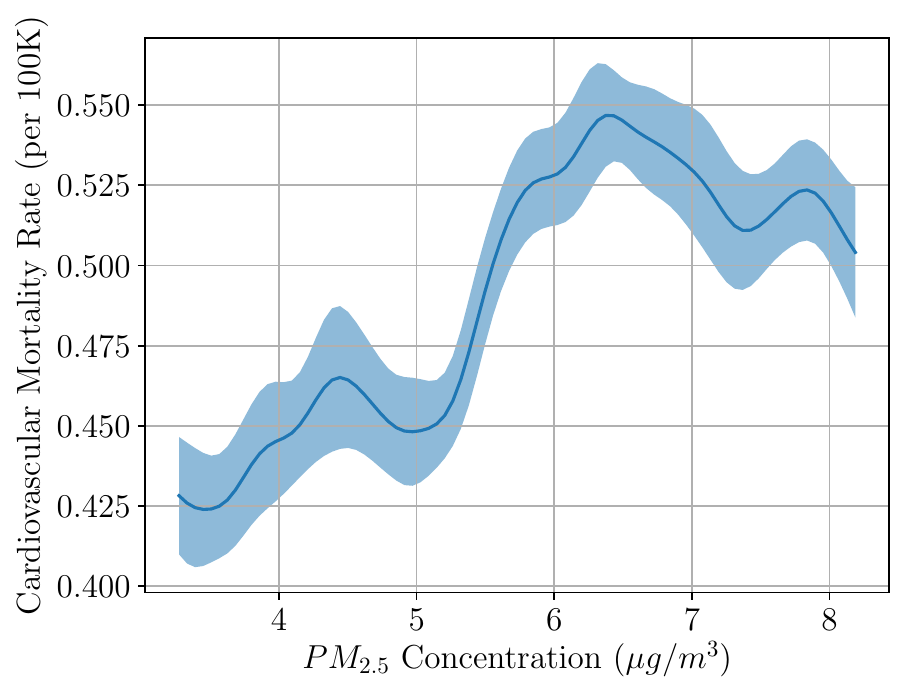}
    \caption{Response Curve}
    \label{fig:nsaph_curve}
    \end{subfigure}
    \quad
    \begin{subfigure}[t]{0.48\textwidth}
    \centering
    \includegraphics[width=\textwidth]{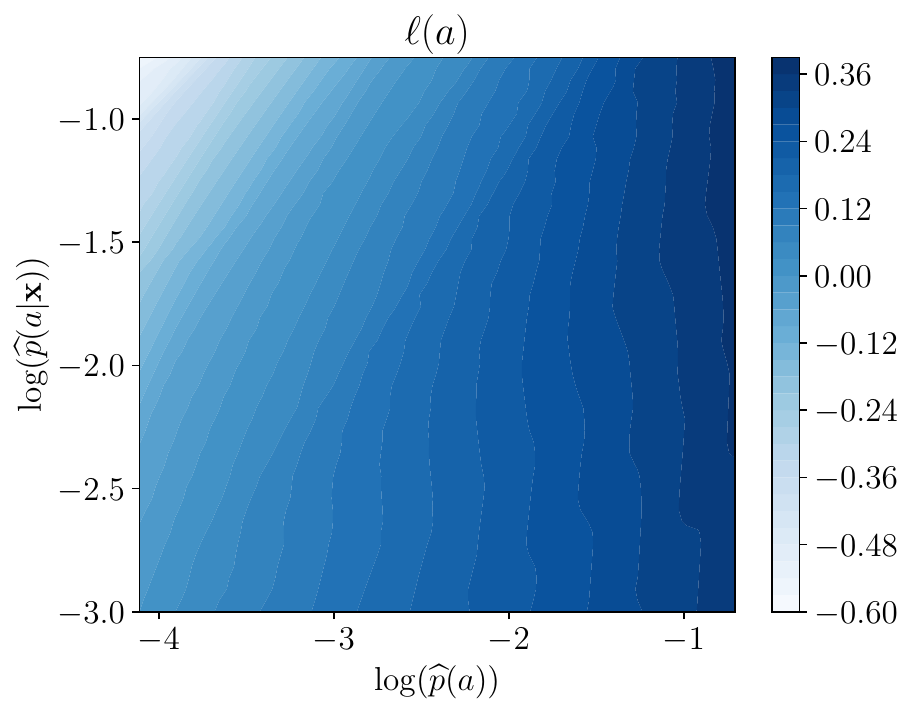}
    \caption{Log-base-weights function}
    \label{fig:nsaph_la}
    \end{subfigure}
    \caption{(\subref{fig:nsaph_curve}) The average treatment effect curve for measuring the impact of $PM_{2.5}$ concentration on the cardiovascular mortality rate. We perform the experiment 100 times and report the mean and $\pm$std range. (\subref{fig:nsaph_la}) The estimated log-base-weight function $\ell_{\bm{\theta}}$ as a function of logarithm of the empirical density of the treatment $\log(\widehat{p}(a))$.  }
    \label{fig:nsaph}
\end{figure*}

As shown in Table \ref{tab:synthetic}, in both linear and non-linear datasets, the E2B is significantly more accurate in uncovering the true treatment response functions. Both constant and IPW base weights perform worse than the base-weights learned by end-to-end balancing.  Comparing EB and E2B, we see that the optimal penalization coefficients that we find are $100$ and $16$, respectively. This indicates that hyperparameter tuning with validation loss ---available only to E2B--- leads to smaller penalization coefficients. This decreased penalization produces stronger covariate balance for the E2B, compared to EB, and thus contributes to the higher accuracy of E2B. 

To gain more insights, in Figure \ref{fig:log-base}, we plot the log-base weight function that we learn as a function of $\log(\widehat{p}(a))$ and $\log(\widehat{p}(a|\bfx))$. We align all curves at their starting point and plot the median of 100 runs. Looking at the y-axis, especially in the linear case, the smaller conditional probability leads to larger base weights. This in effect regularizes the weights to be closer to $\widehat{p}(a|\bfx)$, with a scaling factor. 
Both figures, show variations in the $\log(\widehat{p}(a))$--axis.  This means that the algorithm finds the high-density regions to be more trustworthy for causal inference. Finally, the complexity of the plots emphasizes the need for end-to-end methods for learning weights. Given the results in Figure \ref{fig:log-base}, the learned log-base-weight does not seem to be a convex combination of uniform and SW base weights.

Note that, as \citet{robins2007comment} caution, synthetic data evaluation might exacerbate the extreme weights issue because, unlike when analyzing real data, usually no manual inspection of weights is done.

\paragraph{Mis-specification of Random Response Functions.} To assess the sensitivity of the E2B to the choice of random response functions $\overline{\mu}(a)$, we perform an additional experiment in which the class of random response functions does not include the true response function. In our experiment, the true response function is a nonlinear function (a Hermite polynomial of degree 3) but we use linear random response functions $\overline{\mu}(a)$ for training $\ell_{\bm{\theta}}$. We generate 100 datasets with different randomly selected response functions and report the mean and standard error of RMSE in estimation of the response functions. We also measure the validation error on the datasets generated using the linear response functions.

In our experiments with misspecification, the RMSE is $0.254~(0.013)$. It shows that misspecification of the random response function does not significantly impact the accuracy of E2B and robustness of E2B to misspecification. The result also shows that the gain is not due to the prior knowledge used for generation of the random response functions. Note that the hyperparameter tuning selects the same regularization parameter because during validation we validate with the misspecified model.

 
\subsection{Real Data Experiments}
\label{sec:real}


We study the impact of $PM_{2.5}$ particle level on the cardiovascular mortality rate (CMR) in 2132 counties in the US using the data provided by the National Studies on Air Pollution and Health \citep{Rappold2020}. The data is publicly available under U.S. Public Domain license. The $PM_{2.5}$ particle level and the mortality rate are measured by $\mu g/m^3$ and  the number of annual deaths due to cardiovascular conditions per 100,000 people, respectively. We use only the data for 2010 to simplify the experiment setup; thus we measure the same year impact of $PM_{2.5}$ particle level. Other than the treatment and response variables, the data includes 10 variables such as poverty rate, population, and household income, which we use as confounders. We provide the descriptive statistics and the histograms for the treatment and effect in Appendix \ref{sec:data}.

To train E2B, we create the random dataset (Line \ref{line:rand} in Algorithm \ref{alg:nn}) using Hermite polynomials of max degree 3,  $\mu_y = \left|h_{\bm{\gamma}_{xy}}(\bm{\beta}_{xy}^{\top}\bfx/\|\bm{\beta}_{xy}^{\top}\bfx\|_2) + h_{\bm{\gamma}_{ay}}(\mra)\right|$. We use absolute value to capture the positivity of our response variable.  The data also shows heteroskedasticity; we model the noise as a zero mean Gaussian variable (Line \ref{line:noise} of the algorithm) with variance $\sigma^2(\widehat{y}) = 6.00\widehat{y}$. For regression, we use the non-parametric local kernel regression algorithm. 
We measure the uncertainty in the curves using the deep ensembles technique \citep{lakshminarayanan2017simple} with 100 random ensembles. That means, in each experiment, we initialize the neural network with different random values. To further improve the uncertainty estimation, in each training, we resample the dataset as well.

Figure \ref{fig:nsaph_curve} shows the average treatment effect curve for the impact of $PM_{2.5}$ on CMR. We show the one standard deviation interval using the shaded areas. Starting around $PM_{2.5}=5.3 \mu g/m^3$ the curve increases with a steep slope, confirming the previous studies that increased $PM_{2.5}$ levels increase the probability of cardiovascular mortality. 
We can see that after $PM_{2.5}=6.4 \mu g/m^3$ the curve plateaus and mortality rate stays at elevated levels. Looking at the histogram of the treatments in Figure \ref{fig:hist_pm25} in the appendix, we observe that most counties have $PM_{2.5}$ between 6 and 8. This might justify the fluctuations that we see in this interval and may allude to existence of potential unmeasured confounders. 

Figure \ref{fig:nsaph_la} shows the log-base-weight function that we learn in this data. Similar to the synthetic experiments, we show the median of 100 runs. While the plot shows smaller variations, it is generally inline with the observations we had in the synthetic data.

\section{Discussion}
Causal inference is a well-studied problem. Its main goal is to remove biases due to confounding by balancing the population to look similar to randomized controlled trials. Removing the impact of confounders can play a critical role in reducing and possibly eliminating bias in our decision making leading to potentially positive societal impacts. Our results rely on two classical assumptions: (1) no unmeasured confounders and (2) positivity. While these assumptions are sometimes reasonable in practice, their violation might lead to biased causal inferences. For example, the positivity assumption might be violated if we do not collect any data for a sub-population. Overall, the debiasing property of causal inference should not relieve us from rigorous data collection and analysis. In our experiments, we have been careful to quantify uncertainty in our causal estimation and be wary of over-confidence in our results. We performed our experiments on a CPU machine with 16 cores on AWS, in a region that uses hydroelectric power.

\section{Conclusion}
We observe that, in the entropy balancing framework, the base weights provide an extra degree of freedom to optimize the accuracy of causal inference. We propose end-to-end balancing (E2B) as a technique to learn the base weights such that they directly improve the accuracy of causal inference using end-to-end optimization. In our theoretical analysis we find that E2B weights are approximating Generalized Stable Weights and discuss E2B's statistical consistency. Using synthetic and real-world data, we show that our proposed algorithm outperforms entropy balancing in terms of causal inference accuracy.


\bibliography{references}
\bibliographystyle{icml2022}
\newpage
\clearpage


\clearpage
\appendix
\onecolumn

\section{Proofs of the Theorems}
\label{sec:proofs}
Whenever the context of an expectation operation is not clear, we disambiguate it by specifying the variable that the expectation is taken over and its distribution $\ex_{\bfx \sim f(\bmx)}[\bfx]$.


\subsection{Proof of Theorem \ref{thm:bias}}
\label{sec:proof1}
\begin{proof}
Given that the logarithm function is a strictly increasing function, we can omit it in the optimization; i.e., $\bm{\lambda}^\star = \argmin_{\bm{\lambda}} \ex\left[\exp(\bxg^\top\bm{\lambda} + \ell)\right]$. Because this is an unconstrained optimization, the optimal solution occurs when the gradient is equal to zero.
\begin{align}
    \ex[\bxg\exp\left(\bxg^{\top}\bm{\lambda}^{\star} +\ell\right)] &= \bm{0}, \nonumber\\
    \ex[\bxg w^\star(\mra, \bfx)] &= \bm{0}, \label{eq:step1}
\end{align}
where the last equation is due to the equation of the weights in the population optimization problem (Eq. \ref{eq:pop}). 

Using the definition for the $\bxg$ vector, Eq. (\ref{eq:step1}) implies that $\ex[w^\star(\mra, \bfx)\mra \bm{\phi}(\bfx)] = \bm{0}$. Thus, we conclude that in the weighted population (with distribution $\widetilde{F}$), the $\mra$ and $ \bm{\phi}(\bfx)$ are uncorrelated:
\begin{equation}
    \ex_{(\mra, \bfx) \sim \widetilde{F}}[\mra\bm{\phi}(\bfx)] = \bm{0}\label{eq:uncorr}
\end{equation}

For every set $\mathcal{B} \subset \seta\times \setx$, we can write:
\begin{equation}
    \widetilde{F}(\mathcal{B}) = \int_{\mathcal{B}} w^{\star}(a, \bmx) \mathrm{d}F(a, \bmx).
\end{equation}
The Radon-Nikodym theorem implies that $w^{\star}(a, \bmx)$ is the Radon-Nikodym derivative:
\begin{align}
    w^{\star}(\bm{x}, a) &= \frac{\mathrm{d}\widetilde{F}(\bm{x}, a)}{\mathrm{d}F(\bm{x}, a)} = \frac{\widetilde{f}(\bm{x}, a)}{f(\bm{x}, a)}\\
    &= \frac{\widetilde{f}(\bm{x})\widetilde{f}(a) + \left\{\widetilde{f}(\bm{x}, a) - \widetilde{f}(\bm{x})\widetilde{f}(a) \right\}}{f(\bm{x}, a)},\\
    & = w_{GSW}(a, \bmx) + \frac{\widetilde{f}(\bm{x}, a) - \widetilde{f}(\bm{x})\widetilde{f}(a)}{f(\bm{x}, a)}
\end{align}
Thus, using Eq. (\ref{eq:uncorr}) and Assumptions 1 and 3 we can write
\begin{equation*}
    \sup_{a, \bmx}\left| w^{\star}(a, \bmx) - w_{GSW}(a, \bmx) \right| \leq \delta_{\bm{\phi}_K}/c.
\end{equation*}
\end{proof}

\subsection{Theorem \ref{thm:cons}}
\label{sec:proof2}
\begin{proof}
Given that the logarithm function is a strictly increasing function, we can omit it in the optimizations. Thus the sample and population solutions are:
\begin{align*}
\widehat{\bm{\lambda}}_n &= \argmin_{\bm{\lambda}}  \frac{1}{n}\sum_{i=1}^{n}\exp(\bmg_i^\top\bm{\lambda} + \ell_i), \\
    \bm{\lambda}^\star &= \argmin_{\bm{\lambda}}  \ex\left[\exp(\bxg^\top\bm{\lambda} + \ell)\right].
\end{align*}
The estimator is an M-estimator and given our sample-splitting, the proof follows the asymptotic normality of the estimator \citep[Chapter 5.3]{van2000asymptotic}. 
\begin{equation}
\sqrt{n}\left(\widehat{\bm{\lambda}}_n - \bm{\lambda}^{\star}\right) ~ \overset{d}{\to} ~ \mathcal{N}(\bm{0}, \bm{V}),\label{eq:cons_lam}
\end{equation}
To obtain the value of $\bm{V_1}$, note that the optimal sample solution occurs at the solution of the following equation (Z-estimator equation):
\begin{equation*}
    \sum_{i=1}^{n} \bm{g}_i\exp\left(\bm{g}_i^{\top}\widehat{\bm{\lambda}}_n\right) = \bm{0}.
\end{equation*}
Thus, the score function is $\bm{\psi}_{\bm{\lambda}} = \bm{g}_i\exp\left(\bm{g}_i^{\top}\bm{\lambda}\right)$. We denote the matrix of derivatives of the score function by $\dot{\bm{\psi}}_{\bm{\lambda}}$ whose elements are defined as $\dot{\psi}_{\bm{\lambda}, kk'} =  \sfrac{\partial \psi_{\bm{\lambda}, k}}{\partial \lambda_{k'}}$. Using the theorem in \citep[Chapter 5.3]{van2000asymptotic}, we can write:
\begin{equation}
    \bm{V} = \ex[\dot{\bm{\psi}}_{\bm{\lambda}^{\star}}]^{-1} \ex[\bm{\psi}_{\bm{\lambda}^{\star}}\bm{\psi}_{\bm{\lambda}^{\star}}^{\top}] \ex[\dot{\bm{\psi}}_{\bm{\lambda}^{\star}}]^{-1}.
    \label{eq:vform}
\end{equation}
In the above equation we have assumed that $\ex[\dot{\bm{\psi}}_{\bm{\lambda}^{\star}}]$ matrix is invertable. 
An unbiased sample estimation of $\bm{V}$ can be obtained by substituting $\widehat{\bm{\lambda}}_n$ in place of $\bm{\lambda}^{\star}$ and taking empirical expectations.

An application of the delta method on Eq. (\ref{eq:cons_lam}) yields:
\begin{align}
    \sqrt{n}\left(\frac{\exp\left(\bmg_i^{\top}\widehat{\bm{\lambda}}_n + \ell_i\right)}{\frac{1}{n}\sum_{i=1}^{n}\exp\left(\bmg_i^{\top}\widehat{\bm{\lambda}}_n + \ell_i\right)} - \frac{\exp\left(\bmg_i^{\top}\bm{\lambda}^{\star} + \ell_i\right)}{\frac{1}{n}\sum_{i=1}^{n}\exp\left(\bmg_i^{\top}\bm{\lambda}^{\star} + \ell_i\right)}\right) &  ~ \overset{d}{\to} ~ \mathcal{N}(\bm{0}, \sigma^2),\\
    \sqrt{n}\left(\widehat{w}_n(a_i, \bmx_i) -  \frac{\exp\left(\bmg_i^{\top}\bm{\lambda}^{\star} + \ell_i\right)}{\ex[\exp(\bxg^{\top}\bm{\lambda}^{\star}+\lambda)] }\right)  ~ & \overset{d}{\to} ~ \mathcal{N}(\bm{0}, \sigma^2), \label{eq:slutsky}\\
    \sqrt{n}\left(\widehat{w}_n(a_i, \bmx_i) - w^{\star}(a_i, \bmx_i)\right)  ~ & \overset{d}{\to} ~ \mathcal{N}(\bm{0}, \sigma^2), \label{eq:last_proof2}
\end{align}
where Eq. \eqref{eq:slutsky} is due to Slutsky's theorem and Eq. \eqref{eq:last_proof2} is obtained by substitution of the definition for $w^{\star}(a_i, \bmx_i)$. The variance is obtained by defining the Softmax function $s(\bm{\lambda}) = \frac{\exp\left(\bmg_i^{\top}\bm{\lambda} + \ell_i\right)}{\frac{1}{n}\sum_{i=1}^{n}\exp\left(\bmg_i^{\top}\bm{\lambda} + \ell_i\right)}$. We denote the gradient of the Softmax function by $\nabla s(\bm{\lambda})$. We can write \citep[Chapter 3]{van2000asymptotic}:
\begin{equation*}
    \sigma^2(a_i, \bmx_i) = \nabla s(\bm{\lambda}^{\star})^{\top}\, \bm{V}\, \nabla s(\bm{\lambda}^{\star}).
\end{equation*}
Substituting the value of $\bm{V}$ from \eqref{eq:vform}, we conclude:
\begin{equation*}
    \boxed{\sigma^2(a_i, \bmx_i) = \nabla s(\bm{\lambda}^{\star})^{\top}\, \ex[\dot{\bm{\psi}}_{\bm{\lambda}^{\star}}]^{-1} \ex[\bm{\psi}_{\bm{\lambda}^{\star}}\bm{\psi}_{\bm{\lambda}^{\star}}^{\top}] \ex[\dot{\bm{\psi}}_{\bm{\lambda}^{\star}}]^{-1}\, \nabla s(\bm{\lambda}^{\star}).}
\end{equation*}
Note that the value of the softmax function depends on the value of $(a_i, \bmx_i)$ at each point.
\end{proof}

\section{Neural Network and Training Details}
\label{sec:details}
\subsection{Details of the $\ell_{\theta}$ Neural Network}
The $\ell_{\bm{\theta}}$ network is defined as follows:
\begin{align*}
    \ell_{\bm{\theta}}(z) = cz + \mathrm{dense3}(\,\mathrm{elu}(\,\mathrm{layer\_norm}(\,\mathrm{dense2}(\,\mathrm{tanh}(\,\mathrm{dense1}(\,z\,)\,)\,)\,)\,)\,)
\end{align*}
The linear term $cz$ acts as a skip connection. The input and output dimensions for the dense linear layers are as follows:
\begin{align*}
    \mathrm{dense1}:& 1 \mapsto h,\\
    \mathrm{dense2}:& h \mapsto h,\\
    \mathrm{dense3}:& h \mapsto 1,
\end{align*}
where $h$ denotes the hidden dimension. Because the softmax function is invariant to the constant shifts, we do not have any bias terms for $\mathrm{dense3}$ and the skip connection. $\mathrm{dense2}$ also does not have the bias because of the proceeding layer normalization. The dimension $h$ has been tuned as a hyperparameter on a validation data and set to $10$.
 
\subsection{Details of the Propensity Score Computation for IPW}
We model both $f(a)$ and $f(a|\bmx)$ as univariate normal distributions. This is the correct assumption in our synthetic data. The marginal distribution $f(a)$ is estimated by simply finding the mean and standard deviation of the observed treatment values. For the conditional distribution, we write $\mra | \bfx \sim \mathcal{N}(\mu_a(\bfx), \sigma^2_{a|\bmx})$, where $\mu_a(\bfx)$ is modeled using a feedforward neural network with two layers and $ \sigma^2_{a|\bmx}$ is estimated using the residuals of the neural network predictions. The dimension of the neural network has been tuned as a hyperparameter on validation data and set to $30$.

\subsection{Further Training Details}
We used PyTorch to implement E2B. For reproducibility purposes, we provide the final settings used for training:
\begin{itemize}
    \item Learning algorithm: Adam with learning rate 0.001, no AMSGrad.
    \item Batch size: 1000
    \item Max epochs: 400
    \item Weight decay: $2.5 \times 10^{-5}$.
    \item Validation on a dataset of size 400, every 10 steps.
\end{itemize}

\subsection{Detail of Permutation Weighting}
\label{sec:perm}
We created the stacked data by stacking $\{(\bmx_i, a_i, a_i\odot \bmx_i)\}_{i=1}^{n}$ and $\{(\bmx_i, \widetilde{a}_i, \widetilde{a}_i\odot \bmx_i)\}_{i=1}^{n}$, where $\widetilde{a}$ are permutations of the original treatments. We trained a random forest classifier to predict whether each data is from the permuted or the original set. We tried both random forests and neural networks and obtained better results with the former. We also calibrated the predicted probabilities of the classifier before computation of the weights.

\section{Data and Preprocessing Description}
\label{sec:data}
\begin{landscape}
\begin{table}[t]
    \caption{NSAPH Data Description}
    \label{tab:my_label}
    \centering
    \begin{tabular}{l|c|c|c|c|c|c|c|c|c|c|c|c}
    \toprule
&PM2.5&CMR&healthfac&population&ses&unemploy&HH\_inc&femaleHH&vacant&owner\_occ&eduattain&pctfam\_pover\\
\midrule
count&2132.0&2132.0&2132.0&2132.0&2132.0&2132.0&2132.0&2132.0&2132.0&2132.0&2132.0&2132.0\\
mean&6.17&255.25&0.18&10.78&0.0&7.85&10.69&11.92&14.25&71.44&35.03&11.25\\
std&1.45&56.76&0.5&1.26&0.96&2.83&0.24&3.94&8.71&7.76&7.07&5.2\\
min&2.19&106.14&-2.85&6.2&-1.84&0.0&9.91&2.1&3.8&19.3&9.4&0.0\\
25\%&5.51&215.38&0.0&10.04&-0.67&6.0&10.54&9.3&8.8&67.7&30.4&7.6\\
50\%&6.43&248.16&0.14&10.62&-0.14&7.6&10.67&11.2&11.65&72.7&35.4&10.55\\
75\%&7.15&288.77&0.34&11.46&0.47&9.3&10.83&13.6&16.6&76.7&39.9&13.82\\
max&9.3&557.43&3.33&16.07&6.46&30.9&11.66&38.0&74.0&89.7&54.6&44.9\\
\bottomrule
    \end{tabular}
\end{table}

\begin{figure}[t]
    \centering
    \begin{subfigure}[t]{0.5\textwidth}
    \centering
    \includegraphics[width=\textwidth]{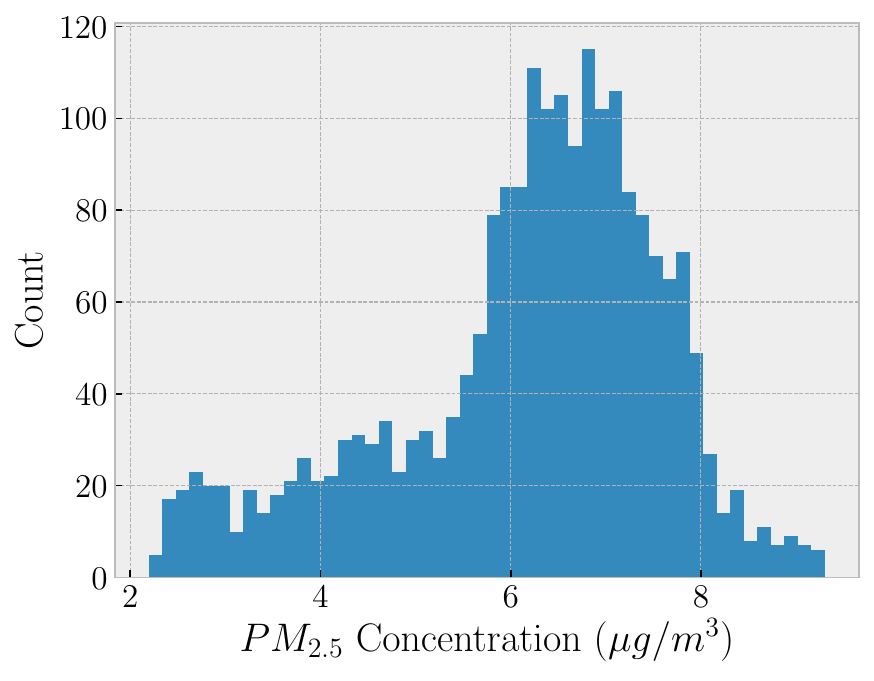}
    \caption{Histogram of $PM_{2.5}$}
    \label{fig:hist_pm25}
    \end{subfigure}
    \quad
    \begin{subfigure}[t]{0.5\textwidth}
    \centering
    \includegraphics[width=\textwidth]{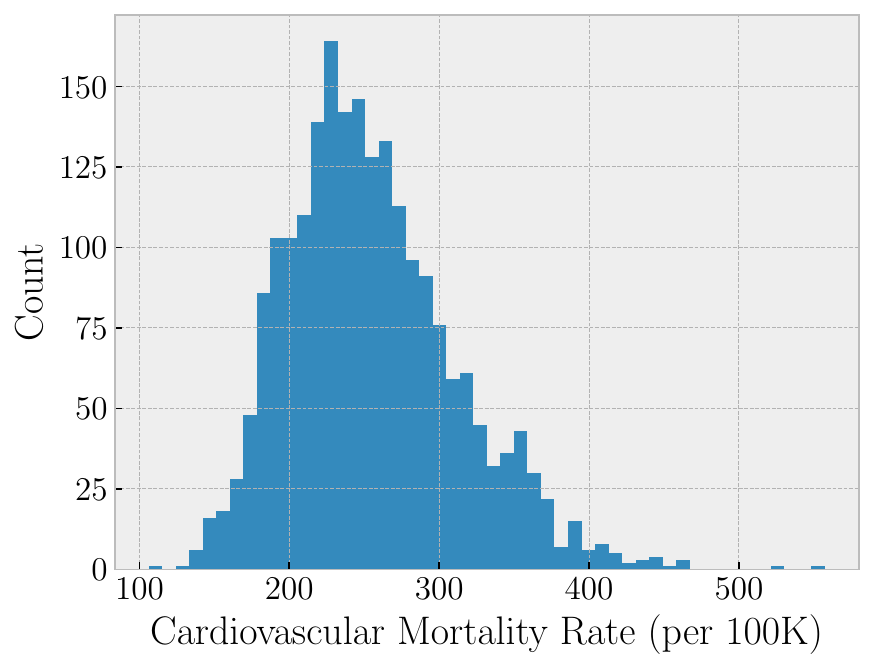}
    \caption{Histogram of Cardiovascular Mortality Rate}
    \label{fig:hist_cmr}
    \end{subfigure}
    \caption{The Histograms of Data}
    \label{fig:hists}
\end{figure}

\end{landscape}

\end{document}